\documentclass[a4paper,10pt]{article}
\usepackage[utf8]{inputenc}
\usepackage{amsmath}
\usepackage{amssymb}
\usepackage{amsthm}
\usepackage{color}
\usepackage{natbib}
\usepackage{url}

\usepackage{titlesec}
\titlelabel{\thetitle.\quad}
\usepackage[english]{babel}
\usepackage{verbatim}
\usepackage{mathtools}
\usepackage{graphicx}
\usepackage{amsmath}
\usepackage{amssymb}
\usepackage{amsfonts}
\usepackage{amsthm}
\usepackage{enumerate}
\usepackage[hang, flushmargin]{footmisc}
\usepackage{hyperref}
\hypersetup{
    bookmarks=true,         
    unicode=false,          
    pdftoolbar=true,        
    pdfmenubar=true,        
    pdffitwindow=false,     
    pdfstartview={FitH},    
    pdftitle={My title},    
    pdfauthor={Author},     
    pdfsubject={Subject},   
    pdfcreator={Creator},   
    pdfproducer={Producer}, 
    pdfkeywords={keyword1} {key2} {key3}, 
    pdfnewwindow=true,      
    colorlinks=true,       
    linkcolor=blue,          
    citecolor=blue,        
    filecolor=magenta,      
    urlcolor=magenta           
}
\usepackage{footnotebackref}
\usepackage[margin=1.0in]{geometry}
\newtheorem{theorem}{Theorem}
\newtheorem{lemma}{Lemma}

\newtheorem{corollary}{Corollary}

\newcommand{\indicator}{{\mathbf 1}}
\newcommand{\cI}{{\mathcal I}}
\newcommand{\cJ}{{\mathcal J}}

\newcommand{\RNum}[1]{\uppercase\expandafter{\romannumeral #1\relax}}
\title{\large Novel Deviation Bounds for Mixture of Independent Bernoulli Variables with Application to the Missing Mass}
\author{\small Bahman Yari Saeed Khanloo \quad \quad bahman.yari@gmail.com}
\date{}
\begin{document}
 
\maketitle

\begin{abstract}
In this paper, we are concerned with obtaining distribution-free concentration inequalities for mixture of independent Bernoulli 
variables that 
incorporate a notion of variance. 
Missing mass is the total probability mass associated to the outcomes that have not been seen in a given sample 
which is an important quantity that connects density estimates obtained from a sample to the population
for discrete distributions on finite or countably inifinite outcomes. 
Therefore, we are specifically motivated to apply our method to study the concentration of missing mass - which can be expressed as a mixture of Bernoulli - in a novel way. 

We not only derive - for the first time - Bernstein-like large deviation bounds for the missing mass
whose exponents behave almost linearly with respect to deviation size, but also sharpen \cite{McAllester2003} and 
\cite{Berend_Kontorovich_Bound}
\color{black}for large sample sizes \color{black}in the case of small deviations (of the mean error) 
which \color{black}is the most interesting case \color{black}in learning theory.
In the meantime, our approach shows that the heterogeneity issue introduced in \cite{McAllester2003} is 
resolvable in the case of missing mass
in the sense that one can use standard inequalities but it may not lead to strong results.
Finally, we postulate that our results are general and can be applied to provide potentially sharp Bernstein-like bounds under some constraints.
\end{abstract}

\section{Introduction}
In this paper, we are interested in bounding the fluctuations of mixture of independent Bernoulli random variables around their mean under specific constraints.  
That is, we fix some finite or countably infinite set $\mathbb{S}$ and let  
$\{ Y_i: i \in \mathbb{S} \}$ be independent Bernoulli variables with
$P(Y_i = 1) = q_i$ and $P(Y_i = 0) = 1-q_i$. 
Moreover, concerning their weights $\{ w_i: i \in \mathbb{S} \}$, we assume that $w_i
\geq 0$ for all $i \in \mathbb{S}$ and $\sum_{i\in \mathbb{S}} w_i = 1$ almost surely. So, we consider the weighted sum:
\begin{equation}
\label{eq:bernoullisum}
Y := \sum_{i \in \mathbb{S}}  w_i Y_i.
\end{equation}
We restrict our attention to cases where both $w_i$ and $q_i$ depend on a given
parameter $n$ - \color{black}usually to be interpreted as `sample size' \color{black}- and we seek to establish bounds of the form
\begin{align}
\mathbb{P}(Y - \mathbb{E}[Y] \leq -\epsilon) \leq \exp(-n \cdot \eta_l (\epsilon)), \nonumber \\
\mathbb{P}(Y - \mathbb{E}[Y] \geq \epsilon) \leq \exp(-n \cdot \eta_u (\epsilon)), \label{eq:bernoullib}
\end{align}
where $\eta_l(\epsilon)$ and $\eta_u(\epsilon)$ are some increasing functions of $\epsilon$ and where it is desirable to find the largest such functions for variable $Y$ and for
the `target' interval of $\epsilon$.  As we will see later, if the $w_i$ and $q_i$ are related to each other and to $n$
in a `specific' way, then it becomes possible to prove such deviation bounds. Further, we will point out that our results can be extended to the missing mass - which has a
similar representation - through association inequalities.

\paragraph{The Challenge and the Remedy}
\cite{McAllester2003} indicate that for weighted Bernoulli sums of the
form (\ref{eq:bernoullisum}), the standard form of Bernstein's
inequality (\ref{eq:bernsteina}) does not lead to
concentration results of form (\ref{eq:bernsteinb}): at least for the
upper deviation of the missing mass, (\ref{eq:bernsteina}) does not
imply any non-trivial bounds of the form (\ref{eq:bernoullib}).  The
reason is essentially the fact that for the missing mass problem,
the $w_i$ can vary wildly --- some can be of order $O(1/n)$, other
$w_i$ may be constants independent of $n$. For similar reasons, other
standard inequalities such as Bennett's, Angluin-Valiant's and
Hoeffding's cannot be used to get bounds on the missing mass of the
form (\ref{eq:bernoullib}) either. Having pointed out the
inadequacy of these standard inequalities, \cite{McAllester2003} do
succeed in giving bounds of the form (\ref{eq:bernoullib}) on the
missing mass, for a function $\eta(\epsilon) \propto \epsilon^2$, both
with a direct argument and using the Kearns-Saul inequality
(\cite{KearnsS98}).  Recently, the constants appearing in the bounds
were refined by \cite{Berend_Kontorovich_Bound}. The bounds proven by
\cite{McAllester2003} and \cite{Berend_Kontorovich_Bound} are qualitatively
similar to Hoeffding bounds for i.i.d. random variables: they do not
improve the functional form from $n \epsilon^2 $ to $n \epsilon$ for small variances.
This leaves open the question whether it is also possible to derive
bounds which are more similar to the Bernstein bound for i.i.d. random
variables (\ref{eq:bernsteinb}) which does exploit variance. In
this paper, we show that the answer is a qualified yes: we give
bounds that depend on weighted variance $\underline{\sigma}^2$ defined in section \ref{sec:defs} rather than
average variance $\bar{\sigma}^2$ as in (\ref{eq:bernsteinb}) which is tight
exactly in the important case when $\bar{\sigma}^2$ is small, and in
which \color{black}the denominator in (\ref{eq:bernsteinb}) is made smaller by a
factor depending on $\epsilon$\color{black}; in the special case of the missing
mass, this factor turns out to be logarithmic in $\epsilon$ and a free parameter $\gamma$ as it will become clear later.

Finally, we derive - using McDiarmid's inequality and Bernstein's inequality
- novel bounds on missing mass that take into account variance

and demonstrate their superiority for standard deviation (STD) size deviations.
\par The key intuition of our approach is that we construct a random variable that is \color{black}less concentrated \color{black}than our
variable of interest but which itself exhibits high concentration for our target deviation size \color{black}when sample size is large\color{black}.
The proofs for mixture of independent Bernoulli variables and missing mass are almost identical; likewise, independence and 
negative association are equivalent when it comes to concentration thanks to the exponential moment method. Therefore, we will just
state our general results for mixture of independent Bernoulli variables along with the required assumptions in section \ref{sec:main_res} and
focus on elaborating on the details for missing mass throughout the rest of the paper \color{black}treating the mixture variables \color{black}as if the comprising variables were independent.

The remainder of the paper is structured as follows. Section \ref{sec:defs} contains notation, definitions and preliminaries. 
Section \ref{sec:main_res} summarizes our main contributions and outlines our competitive results.
In sections \ref{sec:ud_missing_mass} and \ref{sec:ld_missing_mass} we present the proofs of our upper and lower deviation
bounds respectively. Section \ref{sec:analysis} provides a simple
analysis that allows for comparison of our Bernstein-like bounds for missing mass with the existing bounds for the interesting case of STD-sized deviations.
Finally, we briefly mention future work in section \ref{sec:discussion_futurework}.  

\section{Definitions and Preliminaries}
\label{sec:defs}
Consider a fixed but unknown discrete distribution on some finite or countable set $\cI$ and let $\{ w_i: i \in \cI \}$ be the probability of drawing
the $i$-th outcome (i.e. frequency). Moreover, suppose that we observe an i.i.d sample $\{X_j\}_{j=1}^n$ from this distribution. 
Then, missing mass is defined as the total probability mass corresponding to the outcomes that were not present in the given sample.
So, missing mass is a random variable that can be expressed - similar to (\ref{eq:bernoullisum}) - as the following sum:
\begin{equation}
\label{eq:bernoullisum_mm}
Y = \sum_{i \in \cI}  w_i Y_i,
\end{equation}
where we define each $\{ Y_i: i \in \cI \}$ to be a Bernoulli variable that takes on $0$ if the $i$-th outcome exists in the sample and $1$ otherwise
and where we assume that for all $i \in \cI$, $w_i \geq 0$ and $\sum_{i\in \cI} w_i = 1$ \ with probability one.
Denote $P(Y_i = 1) = q_i$ and $P(Y_i = 0) = 1-q_i$ and recall that we assume that $Y_i$s are independent. Therefore, we will have that $q_i=q_i(w_i)=\mathbb{E}[Y_i]=(1-w_i)^n \leq e^{-n w_i}$
where $q_i \in [0,1]$. 
Namely, defining $f:(1,n) \rightarrow (e^{-n}, \frac{1}{e}) \subset (0,1)$ with $f(a)=e^{-a}$ and $a \in D_f$ and taking say $w_i > \frac{a}{n}$ would amount to $q_i(w_i) \leq f(a)$ (c.f.
condition $(a)$ in Theorem \ref{main_theorem}).
This provides a basis for our `thresholding' technique that we will later employ in our proofs.\\

Choosing the representation (\ref{eq:bernoullisum_mm}) for missing mass, one has
\begin{align}
{\mathbb{E}[Y]}_{\cI}=\sum_{i \in \cI} w_i q_i = \sum_{i \in \cI} w_i(1-w_i)^n, \\
{\mathbb{V}[Y]}_{\cI}=\sum_{i \in \cI} w_i^2 q_i (1-q_i)=\sum_{i \in \cI} w_i^2 (1-w_i)^n \big(1-(1-w_i)^n \big), \\
{\underline{\sigma}^2}_{\cI} \coloneqq \sum_{i \in \cI} w_i \text{\sc var\;}[Y_i] = \sum_{i \in \cI} w_i (1-w_i)^n \big(1-(1-w_i)^n \big),
\end{align}
where we have introduced the weighted variance notation ${\underline{\sigma}^2}$ and where each quantity is attached to a set over which it is defined.

One can define the above quantities not just over the set $\cI$ but on some (proper) subset of it that may depend on or be characterized by some variable(s) of interest.
For instance, in our proofs the variable $a$ may be responsible for choosing $\cI_a \subseteq \cI$ over which the above quantities will
be evaluated. For lower deviation and upper deviation, we
find it convenient to denote the associated set by $\mathcal{L}$ and $\mathcal{U}$ respectively. 
Likewise, we will use subscripts $l$ and $u$ to refer to objects that belong to or characterize lower deviation and upper deviation respectively. 

Finally, other notation or definitions may be introduced within the body of the proof when necessarily or when not clear from the context.\\
We will encounter Lambert $W$-function (also known as product logarithm function) in our derivations which describes the inverse relation of $f(W)=W e^W$ and which
can not be expressed in terms of elementary functions. This function is double-valued when defined on real numbers. However, it becomes invertible in restricted domain. 
The lower branch of it is denoted by $W_{-1}(.)$ which is the only branch that \color{black}will be useful \color{black}to us. 
(See \cite{lambert} for a detailed explanation)

\section{Negative Dependence and Information Monotonicity}
\label{neg_dep-info_monot}
Probabilistic analysis of most random variables and specifically the derivation of the majority of probabilistic bounds rely on independence assumption between variables
which offers considerable simplification and convenience.
Many random variables including the missing mass, however, consist of random components that are not independent.

Fortunately, even in cases where independence does not hold, one can still use many standard tools and methods 
provided variables are dependent in some specific ways.
The following notions of dependence are among the common ways that prove useful in these settings:
negative association and negative regression.

\subsection{Negative Dependence and Chernoff's Exponential Moment Method}
Our proof involves variables with a certain type of dependence i.e. negative association. 
One can deduce concentration of sums of negatively associated random variables 
from the concentration of their independent copies thanks to the exponential moment method 
as we shall elaborate later.
This useful property allows us to treat such variables as independent in the context of 
probability inequalities.

In the sequel, we introduce negative association and regression and supply tools that will be essential in our proofs.

\paragraph{Negative Association: }
Any real-valued random variables $X_1$ and $X_2$ are negatively associated
if
\begin{align}
\mathbb{E}[X_1 X_2] \leq \mathbb{E}[X_1] \cdot \mathbb{E}[X_2].
\end{align}
More generally, a set of random variables $X_1,...,X_m$ are negatively associated if for any disjoint subsets $A$ 
and $B$ of the index set $\{1,...,m\}$, we have
\begin{align}
\mathbb{E}[X_i X_j] \leq \mathbb{E}[X_i] \cdot \mathbb{E}[X_j] \quad \text{for} \quad i \in A, \ j \in B.
\end{align}

\paragraph{Stochastic Domination: }
Assume that $X$ and $Y$ are real-valued random variables. Then, $X$ is said to stochastically dominate $Y$ if 
for all $a$ in the range of $X$ and $Y$ 
we have
\begin{align}
P(X \geq a) \geq P(Y \geq a). \label{stochastic_domination}
\end{align}
We use the notation $X \succeq Y$ to reflect (\ref{stochastic_domination}) in short.
\paragraph{Stochastic Monotonicity: }
A random variable $Y$ is stochastically non-decreasing in random variable $X$ if
\begin{align}
x_1 \leq x_2 \ \Longrightarrow \ P(Y|X=x_1) \leq P(Y|X=x_2). \label{stoch_inc}
\end{align}
Similarly, $Y$ is stochastically non-increasing in $X$ if
\begin{align}
x_1 \leq x_2 \ \Longrightarrow \ P(Y|X=x_1) \geq P(Y|X=x_2). \label{stoch_dec}
\end{align}
The notations $(Y|X=x_1) \preceq (Y|X=x_2)$ and $(Y|X=x_1) \succeq (Y|X=x_2)$ represent the 
above definitions using the notion of stochastic domination. Also, we will use shorthands $Y \uparrow X$ and $Y \downarrow X$ 
to refer to the relations described by (\ref{stoch_inc}) and (\ref{stoch_dec}) respectively.

\paragraph{Negative Regression: }
Random variables $X$ and $Y$ have negative regression dependence relation if $X \downarrow Y$.

\cite{Dubhashi-Ranjan} as well as \cite{kumar_frank} summarize numerous useful properties of 
negative association and negative regression. Specifically, the former provides a proposition that indicates that 
Hoeffding-Chernoff bounds apply to sums of 
negatively associated random variables.
Further, \cite{McAllester2003} generalize these observations to essentially any 
concentration result derived based on the exponential moment method by
drawing a connection between deviation probability of a discrete random variable and Chernoff's entropy of a related distribution. 

We provide a self-standing account and prove some of the important results below. 
Also, we shall develop some tools that will be essential in our proofs.
\begin{lemma} \ {\bf [Binary Stochastic Monotonicity]} \label{lemma:monotone_bernoulli} \upshape
Let $Y$ be a binary random variable (Bernoulli) and let $X$ take on values in a totally ordered set $\mathcal{X}$. Then, one has
\begin{align}
Y \downarrow X \ \Longrightarrow \ X \downarrow Y.
\end{align}
\proof Taking any $x$, we have
\begin{align}
P(Y=1 | \ X \leq x) &\geq \inf_{a \leq x} P(Y=1 | \ X=a) \nonumber \\
		    &\geq \sup_{a > x} P(Y=1 | \ X=a) \nonumber \\
		    &\geq P(Y=1 | \ X>x).
\end{align}
The above argument implies that the random variables $Y$ and $\mathbf{1}_{X>x}$ are negatively associated and since the expression  
$P(X>x| \ Y=1) \leq P(X>x | \ Y=0)$ holds for all $x \in \mathcal{X}$, it follows that 
$X \downarrow Y$. $\quad \blacksquare$
\end{lemma}

\begin{lemma} \ {\bf [Independent Binary Negative Regression]} \label{lemma:indep_bin_neg_reg} \upshape
Let $X_1,...,X_m$ be negatively associated random variables and $Y_1,...,Y_m$ be binary random 
variables (Bernoulli) such that either $Y_i \downarrow X_i$ or $Y_i \uparrow X_i$ holds for all $i \in \{1,...,m\}$.
Then $Y_1,...,Y_m$ are negatively associated.
\proof For any disjoint subsets $A$ and $B$ of $\{1,...,m\}$, taking $i \in A$ and $j \in B$ we have
\begin{align}
\mathbb{E}[Y_i Y_j] &= \mathbb{E} \big[ \mathbb{E}[Y_i Y_j| X_1,...,X_m] \big]  \\
&= \mathbb{E} \big[ \mathbb{E}[Y_i|X_i] \cdot \mathbb{E}[Y_j|X_j] \big]  \label{indep} \\
&\leq \mathbb{E} \big[ \mathbb{E}[Y_i|X_i] \big] \cdot \mathbb{E} \big[ \mathbb{E}[Y_j|X_j] \big] \label{neg_ass_ass} \\
& = \mathbb{E}[Y_i] \cdot \mathbb{E}[Y_j].
\end{align}
Here, (\ref{indep}) holds since each $Y_i$ only depends on $X_i$ (independence) and (\ref{neg_ass_ass})
follows because $X_i$ and $X_j$ are negatively associated and we 
have $\mathbb{E}[Y_i|X_i] = P(Y_i|X_i)$. $\blacksquare$
\end{lemma}

\begin{lemma} \ {\bf [Chernoff]} \label{lemma:chernoff_technique} \upshape
For any real-valued random variable $X$ with finite mean $\mathbb{E}[X]$, we have the following for any 
tail $\epsilon>0$ where the entropy $S(X,\epsilon)$ is defined as:
\begin{align}
 &DP(X,\epsilon) \leq e^{-S(X,\epsilon)}, \\
 &S(X,\epsilon) =   \sup_{\lambda} \{ \lambda \epsilon - \ln ( Z(X,\lambda) ) \}  \label{entropy},\\
 &Z(X,\lambda)=\mathbb{E}[e^{\lambda X}] \label{partition_func}.
\end{align}
The lemma follows from the observation that for $\lambda \geq 0$ we have the following
\begin{align}
P(X \geq \epsilon) =  P(e^{\lambda X} \geq e^{\lambda \epsilon}) \leq 
\inf_{\lambda} \frac{\mathbb{E}[e^{\lambda X}]}{e^{\lambda \epsilon}} . \label{exp_moment}
\end{align}
This approach is known as {\em exponential moment method} (\cite{chernoff-exp}) because of the inequality in (\ref{exp_moment}).
\end{lemma}

\begin{lemma} \ {\bf [Negative Association]} \label{lemma:negative_association} \upshape
Deviation probability of sum of a set of negatively associated random variables cannot decrease if independence assumption is imposed.

\proof Let $X_1,...,X_m$ be any set of negatively associated variables. Let $X_1',...,X_m'$ be independent shadow variables, i.e., independent variables such that
$X_i'$ is distributed identically to $X_i$. Let $X=\sum_i^m X_i$ and $X'=\sum_i^m X_i'$. 
For any set of negatively associated 
variables one has $S(X,\epsilon) \geq S(X',\epsilon)$ since:
\begin{align}
Z(X,\lambda) &= \mathbb{E}[e^{\lambda X}] = \mathbb{E}[\prod_i^{m} e^{\lambda X_i}]& \nonumber \\
&\leq \prod_i^{m} \mathbb{E}[e^{\lambda X_i}] = \mathbb{E}[e^{\lambda X'}] = Z(X',\lambda).&
\end{align}
The lemma is due to \cite{McAllester2003} and follows from the definition of 
entropy $S$ given by (\ref{entropy}). $\blacksquare$ \\
This lemma is very useful in the context of probabilistic bounds: it imples that one can treat negatively 
associated variables as if they were independent (\cite{McAllester2003, Dubhashi-Ranjan}).
\end{lemma}

\begin{lemma} \ {\bf [Balls and Bins]} \label{lemma:balls_and_bins} \upshape
Let $\mathbb{S}$ be any sample of $n$ items drawn i.i.d from a fixed distribution on integers $1,...,N$ (bins). Let $C_i$ be the number of times integer $i$
occurs in the sample. The variables $C_1,...,C_N$ are negatively associated.
\proof Let $f$ and $g$ be non-decreasing and non-increasing functions respectively. We have 
\begin{align}
\big( f(x)-f(y) \big) \big( g(x)-g(y) \big) \leq 0. \label{functional}
\end{align}
Further, assume that $X$ is a real-valued random variable and $Y$ 
is an independent shadow variable corresponding to $X$.
Exploiting (\ref{functional}), we obtain
\begin{align}
\mathbb{E}[f(X)g(X)] \leq \mathbb{E}[f(X)] \cdot \mathbb{E}[g(X)], \label{chebychev}
\end{align}
which implies that $f$ and $g$ are negatively associated. Inequality (\ref{chebychev}) is an instance of Chebychev's fundamental {\em association inequality}.

Now, suppose without loss of generality that $N=2$. Let $n$ denote sample size, take $X \in [0,n]$ and consider the following functions
\begin{align}
\left\{
	\begin{array}{ll}
		f(X)=X, \\
		g(X)=n-X,
	\end{array}
\right.
\end{align}
where $n=C_1+C_2$ is the total counts. Since $f$ and $g$ are non-decreasing and non-increasing functions of $X$, choosing $X=f(C_1)=C_1$ we have that
\begin{align}
\mathbb{E}[C_1 \cdot C_2] \leq \mathbb{E}[C_1] \cdot \mathbb{E}[C_2],
\end{align}

which concludes the proof for $N=2$. Now, if we introduce $f(C_i)=C_i$ and $g(C_i)=n-\sum_{j\neq i} C_j$ where $n=\sum_{j=1}^N C_j$, 
for $N>2$ the same argument implies that $C_i$ and $C_j$ are negatively associated for all $j\leq N, \ j \neq i$.
That is to say, any increase in $C_i$ will cause a decrease in some or all of $C_j$ variables with $j \neq i$ and vice versa.
It is easy to verify that the same holds for any disjoint subsets of
the set $\{C_1,...,C_N\}$. $\quad \blacksquare$
\end{lemma}

\begin{lemma} \ {\bf [Monotonicity]} \label{lemma:monotone} \upshape
For any negatively associated random variables $X_1,...,X_m$ and any non-decreasing functions $f_1,...,f_m$,
we have that $f_1(X_1),...,f_m(X_m)$ are negatively associated. The same holds if the functions $f_1,...,f_m$ were non-increasing. 

\textbf{Remark: } The proof is in the same spirit as that of association inequality (\ref{chebychev})
and motivated by composition rules for monotonic functions that one can repeatedly apply to (\ref{functional}).
\end{lemma}

\begin{lemma} \ {\bf [Union]} \label{lemma:union} \upshape
The union of independent sets of negatively associated random variables yields a set of negatively associated random variables.

Suppose that $X$ and $Y$ are independent vectors each of which comprising a negatively associated set. 
Then, the concatenated vector $[X,Y]$ is negatively associated.
\proof Let $[X_1, X_2]$ and $[Y_1, Y_2]$ be some arbitrary partitions of $X$ and $Y$ respectively and assume that
$f$ and $g$ are non-decreasing functions. 

Then, one has
\begin{align}
\mathbb{E}[f(X_1, Y_1)  g(X_2, Y_2) ]= \nonumber \\
\mathbb{E} \big[ \mathbb{E} [f(X_1, Y_1)  g(X_2, Y_2) | \ Y_1,Y_2] \big] \leq \nonumber \\
\mathbb{E}[\mathbb{E} [f(X_1, Y_1) | \ Y_1 ] \mathbb{E} [g(X_2, Y_2) | \ Y_2 ]] \leq \nonumber \\
\mathbb{E}[\mathbb{E} [f(X_1, Y_1) | \ Y_1 ]] \cdot \mathbb{E}[\mathbb{E} [g(X_2, Y_2) | \ Y_2 ]] = \nonumber \\
\mathbb{E}[f(X_1, Y_1)] \cdot \mathbb{E}[g(X_2, Y_2)].
\end{align}
The first inequality is due to independence of $[X_1,X_2]$ from $[Y_1, Y_2]$ which results in negative association being
preserved under conditioning and the second inequality follows because $[Y_1, Y_2]$ are negatively associated (\cite{kumar_frank}). 
The same holds if $f$ and $g$ were non-increasing functions. $\quad \blacksquare$
\end{lemma}

\begin{lemma} \ {\bf [Splitting]} \label{lemma:splitting} \upshape
Splitting an arbitrary bin of any fixed discrete distribution yields a
set of negatively associated random bins.

Let $w=(w_1,...,w_m)$ be a discrete distribution and assume without loss of 
generality that 
$w_i$ is an arbitrary bin of $w$ split into $k$ bins $W_{i1},...,W_{ik}$ such 
that $w_i=\sum_{j=1}^k W_{ij}$.
Then, the random variables $W_{i1},...,W_{ik}$ (random bins) are negatively associated.
Clearly, the same argument holds for any $i \in \{1,...,m\}$ as well as any other 
subset of this set.

\textbf{Remark: } The proof is similar to Lemma \ref{lemma:balls_and_bins}
and based on the observation that each split bin $W_{ij} \propto C_{ij}$ is a 
random variable and they sum to a constant value almost surely.
\end{lemma}

\begin{lemma} \ {\bf [Merging]} \label{lemma:merging} \upshape
Merging any subset of bins of a discrete distribution yields negatively associated random bins.
\proof Let $p=(p_1,...,p_N)$ be a discrete distribution and let $\{C_1,...C_i,...,C_j,...,C_k,...,C_l,...,C_N\}$ 
be the set of count variables. Assume without loss of generality that
$\{C_1,...,C_{ij}^M,...,C_{kl}^M,...,C_N\}$ is the merged set of count 
variables where each $C_{uv}^M$ corresponds 
to a merge count random varlable obtained after merging $p_u$ 
through $p_v$ i.e. $C_{uv}^M=\sum_{t=u}^v C_t$.
The rest of the proof concerns negative association of the variables in the induced set which is identical 
to Lemma \ref{lemma:balls_and_bins} applied to the merged set. $\quad \blacksquare$
\end{lemma}

\begin{lemma} \ {\bf [Absorption]} \label{lemma:absorption} \upshape
Absorbing any subset of bins of a discrete distribution yields negatively associated bins.
\proof Let $p=(p_1,...,p_N)$ be a discrete distribution and let $\{C_1,...,C_N\}$ 
be the set of count variables. Assume without loss of generality that
$\{C_1^A,...,C^A_{N-1}\}$ is the absorb-induced set of count 
variables where $p_N$ has been absorbed to produce $p^A_1,...,p^A_{N-1}$ 
where $p^A_i = p_i+\frac{p_N}{N-1}$
for $i=1,...,N-1$ and where $p_N$ is discarded. 
The rest of the proof concerns negative association of the variables in the induced set which is identical 
to Lemma \ref{lemma:balls_and_bins} applied to the absorbed set. Namely, if a set of variables
are negatively associated, adding a constant to each will preserve their negative association. $\quad \blacksquare$
\end{lemma}

\subsection{Negative Dependence and the Missing Mass}
In the case of missing mass given by (\ref{eq:bernoullisum_mm}), the variables 
$W_i=\frac{C_i}{n}$ are negatively associated owing to Lemma \ref{lemma:balls_and_bins} and linearity of expectation.
Furthermore, each $Y_i$ is negatively associated with $W_i$ and $\forall i: \ Y_i \downarrow W_i$.
Also, $Y_1,...,Y_N$ are negatively associated because they correspond to a set of independent binary variables with negative regression dependence 
(Lemma \ref{lemma:indep_bin_neg_reg}). As a result, concentration variables $Z_1,...,Z_N$ with  
$Z_i=w_i Y_i - \mathbb{E}[w_i Y_i]$ are negatively associated.
 This holds as a consequence of the fact that association inequalities 
 are {\em shift invariant} and for each individual term $w_i Y_i$ we 
 have $W_i Y_i \downarrow W_i$ since $f(w_i)=w_i(1-w_i)^n$ is 
 non-increasing for any $w_i \in (\frac{1}{n+1},1)$. 
Similarly, downward deviation concentration variables $-Z_1,...,-Z_N$ are negatively associated.

\subsection{Information Monotonicity and Partitioning}
\begin{lemma} \ {\bf [Information Monotonicity]} \label{lemma:information_monotone} \upshape
Let $p=(p_1,...,p_t)$ be a discrete probability distribution on $X=(x_1,..,x_t)$ so that $P(X=x_i)=p_i$.
Let us partition X into $m \leq t$ non-empty disjoint groups $G_1,...,G_m$, namely
\begin{align}
X = \cup G_i, \nonumber \\
\forall i \neq j: \ G_i \cap G_j = \emptyset.
\end{align}
This is called {\em coarse binning} since it generates a new distribution with groups $G_i$ whose dimentionality
is less than that of the original distribution. Note that once the distribution is tranformed, 
considering any outcome $x_i$ from the original distribution we will only have access to its group membership information; for instance, we
can observe that it belongs to $G_j$ but we will not be able to recover $p_i$.

Let us denote the induced distribution over the partition $G=(G_1,...,G_m)$ by $p^G=(p_1^G,...,p_m^G)$. Clearly, we have
\begin{align}
p_i^G=P(G_i)=\sum_{j \in G_i} P(x_j).
\end{align}
Now, consider the $f$-divergence $D_f(p^G || q^G)$ between induced probability distributions $p^G$ and $q^G$.
Information monotonicity implies that information is lost as we partition elements of $p$ and $q$ into groups
to produce $p^G$ and $q^G$ respectively. Namely, for any $f$-divergence one has
\begin{align}
D_f(p^G || q^G) \leq D_f(p || q) \label{inf_mono},
\end{align}
which is due to \cite{csiszar_survey, csiszar}.
This inequality is tight if and only if for any outcome $x_i$ and partition $G_j$, we have $p(x_i | G_j) = q(x_i | G_j)$.
\end{lemma}

\begin{lemma} \ {\bf [Partitioning]} \label{lemma:partitioning} \upshape
Partitioning bins of any discrete distribution increases deviation probability of the associated discrete random variable.

Formally, assume that $X$ and $X_{\lambda}$ are discrete random variables defined on the set $\mathcal{X}$ endowed with 
probability distributions $p$ and $p_{\lambda}$ respectively.
Further, suppose that $Y$ and $Y_{\lambda}$ are discrete variables on a partition set $\mathcal{Y}$
endowed with $p^G$ and $p_{\lambda}^G$ that are obtained from $p$ and $p_{\lambda}$ by partitioning using some partition $G$. 
Then, we have
\begin{align}
\forall \epsilon>0: \ DP(X, \epsilon) \leq DP(Y,\epsilon).
\end{align}

\proof Let $\lambda(\epsilon)$ be the optimal $\lambda$ in (\ref{entropy}). Then, 
we have
\begin{align}
S(X,\epsilon) &= \epsilon \lambda(\epsilon) - \ln(Z(X,\lambda(\epsilon))) \nonumber \\ 
&= D_{KL} (p_{\lambda(\epsilon)} || \ p) \nonumber \\
& \geq D_{KL} (p_{\lambda(\epsilon)}^G || \ p^G) \nonumber \\
& = S(Y,\epsilon), \label{entropy_ineq}
\end{align}
where we have introduced the $\lambda$-induced distribution
\begin{align}
P_{\lambda}(X=\epsilon) = \frac{e^{\lambda \epsilon} }{Z(X,\lambda)} P(X=\epsilon).
\end{align}
The inequality step in (\ref{entropy_ineq}) follows from (\ref{inf_mono}) and the observation that
$D_{KL}$ is an instance of $f$-divergence 
where $f(v)=v \ln(v)$ with $v \geq 0$. $\quad \blacksquare$
\end{lemma}

\section{Main Results}
\label{sec:main_res}
We prove bounds of the form (\ref{eq:bernoullib}) if $n$, $w_i$ and
$q_i$ are related - as mentioned above - via a function $f$ which is a parameter of the
problem. Our main results are outlined below.

\begin{theorem}
\label{main_theorem}
\upshape
Let $f: (1,n) \rightarrow (0,1)$ be some strictly decreasing function, $a \in D_f$ some threshold variable
and $n> 0$ be a fixed integer. Further, let $q^\circ: (0,1) \rightarrow (0,1)$ be
some function such that for all $i \in \cI$, $q_i = q^\circ(w_i)$ 
and for all $1 < a < n$ and all $0 < w \leq 1$, 
  the condition ``$(a)$: either $w \leq a/n$ or $q^{\circ}(w)
  \leq f(a)$ or both'' holds. \color{black}Moreover, assume that for any $w_1,...,w_t >
0$ with $w= \sum_{i=1}^t w_i$, $q^{\circ}$ is such that the additional condition ``$(b)$: $q(w) \leq \prod_{i=1}^t q(w_i)$'' holds.\color{black}
\begin{enumerate} 
\upshape
\item Suppose that there exists a function $q^\circ$ as
  described above. 
Then, for any $0<\epsilon<1$ we obtain
\begin{align}
\mathbb{P}(Y - {\mathbb E} [Y] \geq \epsilon)
\leq \inf_{\gamma} \left\{ 
\exp\left(- C_1 \cdot \frac{n \epsilon^2 {(\gamma-1)}^2}{\underline{\sigma}^2_{\mathcal{U}} \cdot f^{-1}(\epsilon/\gamma) \cdot \gamma^2}  \right) \right\},
\end{align}
where $C_1$ is a constant and $\gamma \in D_{\gamma}$ is a problem-dependent free parameter which is to be optimized in order
to determine problem-dependent set $\mathcal{U}\subseteq \cI$ as well
as the optimal threshold $a$.
\item Assume that there exists a function $q^\circ$ as
  above. Then, for $0<\epsilon<1$ we have
\begin{align}
\mathbb{P}(Y - {\mathbb E} [Y] \leq -\epsilon)
\leq \inf_{\gamma} \left\{ 
\exp\left(- C_2 \cdot \frac{n \epsilon^2 {(\gamma-1)}^2}{\underline{\sigma}^2_{\mathcal{L}} \cdot f^{-1}(\epsilon/\gamma) \cdot \gamma^2}  \right) \right\},
\end{align}
where $C_2$ is a constant and $\gamma \in D_{\gamma}$ is again a free parameter that determines $\mathcal{L}\subseteq \cI$ and controls thresholding variable $a$.

\end{enumerate}
\end{theorem}
By applying union bound to the above theorem, we immediately
obtain the following corollary.
\begin{corollary}
\upshape
Assume that conditions (a) and (b) as above hold for some variable $Y$. Then, for any given $0<\epsilon<1$  we will have
\begin{align}
\mathbb{P}(\lvert Y - \mathbb{E} [Y] \rvert \geq \epsilon) \leq 
2 \inf_{\gamma} \left\{ 
\exp\left(- \frac{\min\{C_1,C_2\} \cdot n \epsilon^2 {(\gamma-1)}^2}{\max \{ \underline{\sigma}^2_{\mathcal{L}},\underline{\sigma}^2_{\mathcal{U}} \}
\cdot f^{-1}(\epsilon/\gamma) \cdot \gamma^2}  \right) \right\}.
\end{align}
\end{corollary}

\begin{corollary}
\upshape
The above deviation bounds hold for any mixture variable $Y$ with each $Y_i \in [0,1]$ if we have that $\mathbb{E}[Y_i] \leq q_i$ for all $i$.
The proof of this generalization is provided in appendix \ref{generalization_to_bounded}.
\end{corollary}
\begin{corollary}
\upshape
Observe that the $Y_i$s are negatively associated in the case of missing mass.
Also, recall that $0 \leq w_i \leq 1$ for all $i$ which gives $w_i Y_i \leq Y_i$ for all $i$.
Thus, combining lemma $5$ and lemma $7$ in \cite{McAllester2003} imples that the above bounds extend to the missing mass.
\end{corollary}
\color{black}In the missing mass problem, we choose $f(a) = e^{-a}$ where $a$ is a threshold variable set by our optimization procedure and $n$ is the sample size. 
If $Y$ is the missing mass, our elimination procedure guarantees that condition $(a)$ would hold for $q_i=\mathbb{E}[Y_i]$ (see section \ref{sec:ud_missing_mass}).
On the other hand, the split condition $(b)$ holds for $Y$
as well (see appendix \ref{app:prove_decomp}). Our results are summarized below.\color{black}
\begin{theorem}
\upshape
Let $Y$ denote the missing mass. Then, we have the following bounds.

\noindent (\RNum{1}): 
In the case of upward deviation, we obtain as in section \ref{sec:ud_missing_mass} for any $0<\epsilon<1$ the bound
\begin{equation}\label{eq:mm_bound_ud}
\mathbb{P}(Y - {\mathbb E} [Y] \geq \epsilon )
\leq e^{-\frac{3}{4} c(\epsilon) \cdot n \epsilon},
\end{equation} 
where $c(\epsilon)=\frac{\gamma_{\epsilon}-1}{\gamma_{\epsilon}^2}$ and $\gamma_{\epsilon}=-2W_{-1} \big(-\frac{\epsilon}{2 \sqrt e} \big)$. \\

Similarly, we obtain the following upward deviation bound whose exponent is quadratic in $\epsilon$: 
\begin{align}
\mathbb{P}(Y - {\mathbb E} [Y] \geq \epsilon )
\leq e^{-\color{black}4\color{black}c(\epsilon) \cdot n \epsilon^2}.\label{eq:quad_bound_ud}
\end{align}
Inequality (\ref{eq:quad_bound_ud}) sharpens (\ref{eq:mm_bound_ud}) for all $\color{black}0.187\color{black}<\epsilon<1$.

\noindent (\RNum{2}): In the case of downward deviation, we obtain (as in section \ref{sec:ld_missing_mass}) for any $\color{black}0 < \color{black}\epsilon < 1$ the bound
\begin{align}\label{eq:mm_bound_ld1}
\mathbb{P}(Y - {\mathbb E} [Y] \leq -\epsilon )
\leq e^{-\frac{3}{4}c(\epsilon) \cdot n \epsilon}. 
\end{align}

\noindent Also, we obtain the following downward deviation bound whose exponent is quadratic in $\epsilon$: 
\begin{align}
\mathbb{P}(Y - {\mathbb E} [Y] \leq -\epsilon )
\leq e^{-4c(\epsilon) \cdot n \epsilon^2}. \label{eq:quad_bound_ld}
\end{align}
Inequality (\ref{eq:quad_bound_ld}) sharpens (\ref{eq:mm_bound_ld1}) for all $\color{black}0.187\color{black}<\epsilon<1$.

\end{theorem}
\color{black}In general cases other than the missing mass, as long as our
conditions hold for some function $f$, we obtain Bernstein-like
inequalities. Furthermore, in the special case of missing mass, we show in our proof below that $\underline{\sigma}^2_{\mathcal{S}} \leq \epsilon$ for a suitable choice
of $\mathcal{S} \subset \cI$. That is to say, we derive Bernstein-like deviation bounds whose exponents depend almost linearly on $\epsilon$ and which are
sharp for small $\epsilon$.\color{black}

\section{Proof for Upper Deviation Bounds}
\label{sec:ud_missing_mass}
\par The idea of the proof is to reduce the problem to one in which all
weights \color{black}smaller \color{black}than the threshold $\tau=\frac{a}{n}$ are eliminated, where $a$ will depend on $\gamma$ and the
$\epsilon$ of interest. These are exactly the weights that cause the heterogeneity issue noted by \cite{McAllester2003}. 
The reduction is done by discarding 
the weights that are \color{black}smaller \color{black}than $\tau$, 
namely setting the corresponding $Y_i$ to $0$ 
and adding a compensation term - that depends on $\gamma$ and $\epsilon$- to $\epsilon$. \color{black}Finally, we choose a threshold
that yields optimal bounds: interestingly, the optimal threshold will turn out to be a function of $\epsilon$.\color{black}

Let $\cI_a$ denote the subset of $\cI$ with $\color{black}w_i < \frac{a}{n}$ \color{black}and
$\cI_b = \cI \setminus \cI_a$. For each $i \in \cI_b$ and for
some $k \in \mathbb{N}$ that depends on $i$ (but we suppress that notation
  below), we will have that $k \cdot \frac{a}{n} \leq w_i < (k+1) \cdot \frac{a}{n}$. For all such
$i$, we define the additional Bernoulli random variables $Y_{ij}$ with
$j \in \cJ_i := \{1, \ldots, k \}$ and their associated weights. For
$j \in \{1, \ldots, k-1\}$, $w_{ij} = \frac{a}{n}$ and $w_{ik} = w_i - (k-1) \cdot \frac{a}{n}$.
In this way, all weights that are larger than $\frac{a}{n}$ are
split up into $k$ weights, each of which is in-between $\frac{a}{n}$ and $\frac{2a}{n}$
(more precisely, the first $k-1$ ones are exactly $\frac{a}{n}$, the
latter one may be larger).

We now consider the random variable \color{black}$Y' = 
\sum_{i \in \cI_b, j \in \cJ_i} w_{ij} Y_{ij}$ \color{black}and define $\mathcal{U}=\{i \in \cI_b: \color{black}\tau \leq \ w_{[i]} < 2\tau \color{black}\}$ (where we drop $j$ in the subscript).

Now, by choosing $a$ such that $f(a)= e^{-a}=\frac{\epsilon}{\gamma}$ so that
$a= f^{-1}(\frac{\epsilon}{\gamma})=\log(\frac{\gamma}{\epsilon})$ for any $0<\epsilon<1$ and $ e \epsilon < \gamma < e^n \epsilon$, the upper deviation
bound for the missing mass can be derived as follows
\begin{align}
&\mathbb{P}(Y  - {\mathbb E}[Y] \geq \epsilon) \leq  \\
& \mathbb{P}(Y' - {\mathbb E}[Y] \geq \epsilon) = \label{eq:seq_decomp}\\
& \mathbb{P}(Y' - {\mathbb E}[Y'] + \left( {\mathbb E}[Y'] - {\mathbb E}[Y]\right) \geq \epsilon) \leq 
  \mathbb{P}(Y'  - {\mathbb E}[Y'] + f(a) \geq \epsilon)=& \label{eq:compensation_u} \\
&\mathbb{P} \Big( \color{black}Y' - {\mathbb E}[Y'] \color{black} \geq (\frac{\gamma-1}{\gamma})\epsilon \Big)\leq \label{int_gamma} & \\
&\exp \left(- \frac{{(\frac{\gamma-1}{\gamma})}^2 \epsilon^2}{2 (\mathbb{V}_{\mathcal{U}} + \frac{\alpha_u}{3} \cdot
					      (\frac{\gamma-1}{\gamma}) \cdot \epsilon)} \right) \leq &  \label{apply_bernstein_ud}\\
&\leq \exp\left(- \frac{{(\frac{\gamma-1}{\gamma})}^2 \epsilon^2}{2 (\color{black}\frac{a}{n} \color{black}\cdot \epsilon + \frac{2a}{3n} \cdot
					      (\frac{\gamma-1}{\gamma}) \cdot \epsilon)} \right) \leq \label{replace_bound_variance12}
\inf_{1 < \gamma < e^n } \Big\{ \exp\left(- \frac{3 n \epsilon (\gamma-1)^2}{\color{black}8 \color{black}\gamma^2 \log(\frac{\gamma}{\epsilon})} \right) \Big\} =&\\
&e^{- c(\epsilon) \cdot n \epsilon},& \label{ud_linear}
\end{align}
where $c(\epsilon)=\frac{3 (\gamma_{\epsilon}-1)}{\color{black}4 \color{black}\gamma_{\epsilon}^2}$ 
and $\gamma_{\epsilon}=-2W_{-1} \Big(-\frac{\epsilon}{2 \sqrt e} \Big)$. Clearly, we will have that 
$\tau_{\text{opt}}=\frac{a_{\text{opt}}}{n}$ where $a_{\text{opt}}=\log(\frac{\gamma_\epsilon}{\epsilon})$.

\noindent \noindent The proof for inequality (\ref{eq:seq_decomp}) is provided in appendix \ref{app:prove_decomp}. 
Inequality (\ref{eq:compensation_u}) follows because the compensation term will remain small, namely
\begin{align}
g_u(\epsilon) = \mathbb{E}[Y'] - \mathbb{E}[Y] = \sum_{i \in \cI_b} \sum_{j \in \cJ_i} w_{ij} q_{ij} - \sum_{i \in \color{black}\cI} \color{black}w_i q_i
\leq  \sum_{i \in \cI_b} \sum_{j \in \cJ_i} w_{ij} q_{ij} \\
\leq \sum_{i \in \cI_b} \sum_{j \in \cJ_i} w_{ij} f(a) \leq f(a). \label{split_req}
\end{align}

To see why (\ref{split_req}) holds, it is sufficient to recall that $q_{ij}=q(w_{ij})$ and all $w_{ij}$s are greater than or equal to $\frac{a}{n}$.
Inequality (\ref{apply_bernstein_ud}) is Bernstein's inequality applied to the random variable $Z = \sum_{i \in \mathcal{U}}  Z_i$ with $Z_i =  w_i Y_i - \mathbb{E}[w_iY_i]$
where we have chosen \color{black}$\alpha_u=2\tau$\color{black}. 

In order to derive the upper bound on $\mathbb{V}_{\mathcal{U}}$ we first need to specify $\mathcal{U}$.
Here, we will consider the set $\mathcal{U}=\cI_b$ (as characterized above) which is the set of weights we obtain after splitting indexed by
$i$ in what follows below again for simplicity of notation. 
Observe that the functions $f(x)=x(1-x)^n$ and \color{black}$f(x)=x^2(1-x)^n$ \color{black}are decreasing on 
$(\frac{1}{n+1},1)$ and \color{black}$(\frac{2}{n+2},1)$ \color{black}respectively. 
Thus, for $1<a<n$ and for any $0<\epsilon<\color{black}1$\color{black}, 
the upper bound can be expressed as
\begin{align}
\mathbb{V}_{\mathcal{U}}{(a,n)} = \sum_{i:  \color{black}a/n \leq w_i < 2a/n\color{black}; \ \sum_i w_i \leq 1} w_i ^2 {(1-w_i)}^n \Big(1-{(1-w_i)}^n \Big) \leq \\
\color{black}\frac{a}{n} \color{black}\cdot \sum_{i:  \color{black}a/n \leq w_i < 2a/n\color{black}; \ \sum_i w_i \leq 1} w_i {(1-w_i)}^n \Big(1-{(1-w_i)}^n \Big) =
\color{black}\frac{a}{n} \color{black} \cdot {\underline{\sigma}}^2_{\mathcal{L}} \leq \\
\color{black}\frac{a}{n} \color{black}\cdot \sum_{i:  \color{black}a/n \leq w_i < 2a/n\color{black}; \ \sum_i w_i \leq 1}  w_i {(1-w_i)}^n =
\color{black}\frac{a}{n} \color{black}\cdot \mathbb{E}_{\mathcal{L}}(a,n)
\leq \\
\color{black}\frac{a}{n} \color{black}\cdot \sum_{i: a/n \color{black}\leq \color{black}w_i < 2a/n; \ \sum_i w_i \leq 1} w_i {(1-w_i)}^n \leq  \\
\color{black}\frac{a}{n} \color{black}\cdot \underbrace{|\cI_{(a,n)}|}_{\leq \frac{n}{a}} \cdot \frac{a}{n} \Big( 1-\frac{a}{n}\Big)^n
\leq \color{black}\frac{a}{n} \color{black}\cdot e^{-a} < \color{black}\frac{a}{n} \color{black}\cdot \epsilon. \label{upperbound_variance_mm2}
\end{align}

\noindent Now, if we choose to apply McDiarmid's inequality in the form (\ref{mcdiarmid_C}), we would skip the splitting procedure and redefine 
$Y'_i =  \min \{Y_i, \indicator_{[w_i \leq \tau]} \}$
and set $\mathcal{U}=\cI_a$ 
so that we can continue the proof from (\ref{int_gamma}) and write
\begin{align}
&\mathbb{P} \Big( Y' - {\mathbb E}[Y'] \geq (\frac{\gamma-1}{\gamma})\epsilon \Big)
\leq \exp\left(- \frac{2 {(\frac{\gamma-1}{\gamma})}^2 \epsilon^2}{C_{\mathcal{U}}} \right) \leq &\\
&\exp\left(- \frac{\color{black}2\color{black} n \epsilon^2 (\gamma-1)^2}{\gamma^2 \cdot f^{-1}(\frac{\epsilon}{\gamma})} \right) \leq & \\
&\inf_{1 < \gamma < e^n} \Big\{
\exp\left(- \frac{\color{black}2\color{black}n \epsilon^2 (\gamma-1)^2}{\gamma^2 \cdot \log(\frac{\gamma}{\epsilon})} \right) \Big\} =&\\
&e^{- c(\epsilon) \cdot n \epsilon^2}, &
\end{align}
where $c(\epsilon)=\frac{\color{black}4\color{black} (\gamma_{\epsilon}-1)}{\gamma_{\epsilon}^2}$. 

Here, we are required to repeat what we performed in (\ref{split_req}) by taking $\mathbb{E} [Y'_i] = q'_i$ with $q'_i =
q_i$ if $w_i \color{black} \leq \color{black}\tau$ and $q'_i = 0$ otherwise, so that we have
\begin{align}
g_l(\epsilon)={\mathbb
  E}[Y'] - {\mathbb E}[Y] 
= \sum_{i \in \cI} w_i (q'_i  - q_i) =  \sum_{i: w_i
  \color{black} \leq \color{black}a/n} w_i q_i - \sum_{i \in \cI} w_i q_i \color{black}\leq \\
\sum_{i: w_i \color{black} \leq \color{black}a/n} \color{black}w_i q_i \leq \sum_{i: w_i \color{black} \leq \color{black}a/n} w_i f(a) \leq f(a).
\end{align}

\noindent As for upper bound on $C_{\mathcal{U}}=\sum_{i \in \mathcal{U}} c_i^2$, we have
\begin{align}
C_{\mathcal{U}}=\sum_{i: w_i \leq a/n} w_i^2 \leq \frac{a}{n} \cdot \sum_{i: w_i \leq a/n} w_i \leq \frac{a}{n} \cdot \sum_{i \in \cI} w_i \leq \frac{a}{n}.
\end{align}

\noindent Note that utilizing $C_{\mathcal{U}}$ leads to a sharper bound for $\color{black}0.187\color{black}<\epsilon<1$. 

\section{Proof for Lower Deviation Bounds}
\label{sec:ld_missing_mass}

The proof proceeds in the same spirit as section \ref{sec:ud_missing_mass}. The idea is again to
reduce the problem to one in which all weights \color{black}smaller \color{black}than threshold \color{black}$\tau=\frac{a}{n}$ \color{black}are eliminated.

So, we define $Y'_i =  \min \{Y_i, \indicator_{[\color{black}w_i > \tau\color{black}]} \}$ and $Y' = \sum w_i Y'_i$. 

\color{black}Also here, the weights that are larger than $\tau$ are split to enable us shrink the variance while controlling
the magnitude of each term (and consequently the constansts) before the application of the main inequality takes place. \color{black}

Thus, we consider subsets $\cI_a$ and $\cI_b$ as before and define the set $\mathcal{L}=\{i \in \cI_b: \color{black}\tau \leq \ w_{[i]} < 2\tau \color{black}\}$ 
which again consists of the set of weights we obtain after splitting and introduce the random variable $\color{black}Y'''\color{black}= \sum_{i \in \mathcal{L}} w_i Y_i$.

By choosing $a$ such that $f(a)= e^{-a}=\frac{\epsilon}{\gamma}$ so that
$a= f^{-1}(\frac{\epsilon}{\gamma})=\log(\frac{\gamma}{\epsilon})$, for any $0<\epsilon<1$ with $ e \epsilon < \gamma < e^n \epsilon$ 
we obtain a lower deviation bound for missing mass as follows
\begin{align}
&\mathbb{P}(Y  - {\mathbb E}[Y] \leq -\epsilon) \leq& \\
& \mathbb{P}(Y' - {\mathbb E}[Y] \leq -\epsilon) =& \\
& \mathbb{P}(Y' - {\mathbb E}[Y'] + \left( {\mathbb E}[Y'] - {\mathbb E}[Y]\right) \leq -\epsilon) \leq 
 \mathbb{P}(Y'  - {\mathbb E}[Y'] - f(a) \leq -\epsilon) =& \label{eq:compensation_l}\\
&\mathbb{P} \Big( Y'-\mathbb{E}[Y'] \leq -(\frac{\gamma-1}{\gamma})\epsilon \Big) \color{black}\leq  \label{introduce_gamma4}
\color{black}\mathbb{P} \Big( \color{black}Y''' - {\mathbb E}[Y'''] \color{black}\leq -(\frac{\gamma-1}{\gamma})\epsilon \Big)&  \\
&\leq \exp\left(- \frac{{(\frac{\gamma-1}{\gamma})}^2 \epsilon^2}{2 (\mathbb{V}_{\mathcal{L}} + \frac{\alpha_l}{3} \cdot
					      (\frac{\gamma-1}{\gamma}) \cdot \epsilon)} \right) \leq & \label{apply_bernstein_ld} \\
&\leq \exp\left(- \frac{{(\frac{\gamma-1}{\gamma})}^2 \epsilon^2}{2 (\frac{a}{n} \cdot \epsilon + \frac{2a}{3n} \cdot
					      (\frac{\gamma-1}{\gamma}) \cdot \epsilon)} \right) \leq & \\
&\inf_{1 < \gamma < e^n} \Big\{
\exp \Big(- \frac{3}{8}n \cdot \underbrace{\frac{\epsilon (\gamma-1)^2}{\gamma^2 \cdot \log(\frac{\gamma}{\epsilon})}}_{\phi(\gamma, \epsilon)} \Big) \Big\} =& \label{phi_func_l}\\
&e^{- c(\epsilon) \cdot n \epsilon}, & \label{ld_linear}
\end{align}
where $c(\epsilon)=\frac{3 (\gamma_{\epsilon}-1)}{\color{black}4 \color{black}\gamma_{\epsilon}^2}$ 
and $\gamma_{\epsilon}=-2W_{-1} \big(-\frac{\epsilon}{2 \sqrt e} \big)$ 
\color{black}and $\tau_{\text{opt}}$ is as before. \color{black}The first inequality follows because we have $Y' \leq Y$.
Inequality (\ref{eq:compensation_l}) follows since $\mathbb{E} [Y'_i] = q'_i$ where $q'_i =
q_i$ if $w_i \color{black}> \color{black}\tau$ and $q'_i = 0$ otherwise, so that by exploiting condition $(a)$ we can write
\begin{align}
g_l(\epsilon)={\mathbb
  E}[Y'] - {\mathbb E}[Y] 
= \sum_{i \in \cI} w_i (q'_i  - q_i) =  \sum_{i: w_i
  \color{black}> \color{black}a/n} w_i q_i - \sum_{i \in \cI} w_i q_i =\\
-\sum_{i: w_i \color{black}\leq \color{black}a/n} w_i q_i \geq -\sum_{i: w_i \color{black}\leq \color{black}a/n} w_i f(a) \geq -f(a).
\end{align}
The proof for inequality (\ref{introduce_gamma4}) is based on the split condition $(b)$ similar to (\ref{eq:seq_decomp}). The difference here is that
we consider $Y'$ and $Y'''$ instead and we need to set $t=(\frac{\gamma-1}{\gamma}) \epsilon$. Deviation probability is
decreasing in absolute deviation size and the expected value of missing mass will only grow after splitting i.e. $\mathbb{E}[Y''']>\mathbb{E}[Y']$ which is again due to condition $(b)$.

Inequality (\ref{apply_bernstein_ld}) is Bernstein's inequality applied to the random variable $Z = \sum_{i \in \mathcal{L}}  Z_i$ with $Z_i =  w_i Y_i - \mathbb{E}[w_iY_i]$ and  
we have set $\alpha_l=2\tau$.
The derivation of upperbound on $\mathbb{V}_{\mathcal{L}}$ is exactly identical to that of $\mathbb{V}_{\mathcal{U}}$.

\noindent If we employ McDiarmid's inequality in the form (\ref{mcdiarmid_C}), we can skip splitting procedure and redefine $Y'_i =  \min \{Y_i, \indicator_{[w_i \leq \tau]} \}$
and set $\mathcal{L}=\cI_a$ 
so we can follow the proof from (\ref{introduce_gamma4}) and write
\begin{align}
&\mathbb{P} \Big( Y' - {\mathbb E}[Y'] \leq -(\frac{\gamma-1}{\gamma})\epsilon \Big)
\leq \exp\left(- \frac{2 {(\frac{\gamma-1}{\gamma})}^2 \epsilon^2}{C_{\mathcal{L}}} \right) \leq &\\
&\exp\left(- \frac{\color{black}2\color{black} n \epsilon^2 (\gamma-1)^2}{\gamma^2 \cdot f^{-1}(\frac{\epsilon}{\gamma})} \right) \leq & \\
&\inf_{1 < \gamma < e^n} \Big\{
\exp\left(- \frac{\color{black}2\color{black}n \epsilon^2 (\gamma-1)^2}{\gamma^2 \cdot \log(\frac{\gamma}{\epsilon})} \right) \Big\} =&\\
&e^{- c(\epsilon) \cdot n \epsilon^2}, &
\end{align}
where $c(\epsilon)=\frac{\color{black}4\color{black} (\gamma_{\epsilon}-1)}{\gamma_{\epsilon}^2}$. 

Now, we need to repeat what we did in (\ref{eq:compensation_l}) by taking $\mathbb{E} [Y'_i] = q'_i$ with $q'_i =
q_i$ if $w_i \color{black} \leq \color{black}\tau$ and $q'_i = 0$ otherwise, so that we have
\begin{align}
g_l(\epsilon)={\mathbb
  E}[Y'] - {\mathbb E}[Y] 
= \sum_{i \in \cI} w_i (q'_i  - q_i) =  \sum_{i: w_i
  \color{black} \leq \color{black}a/n} w_i q_i - \sum_{i \in \cI} w_i q_i = \\
-\sum_{i: w_i \color{black}> \color{black}a/n} w_i q_i \geq -\sum_{i: w_i \color{black}> \color{black}a/n} w_i f(a) \geq -f(a).
\end{align}

\noindent As for upper bound on $C_{\mathcal{L}}=\sum_{i \in \mathcal{L}} c_i^2$, we have
\begin{align}
C_{\mathcal{L}}=\sum_{i: w_i \leq a/n} w_i^2 \leq \frac{a}{n} \cdot \sum_{i: w_i \leq a/n} w_i \leq \frac{a}{n} \cdot \sum_{i \in \cI} w_i \leq \frac{a}{n}.
\end{align}

\noindent Note that working with $C_{\mathcal{L}}$ again leads to a sharper bound for $\color{black}0.187\color{black}<\epsilon<1$. 

\section{\large Comparison of Bounds on Missing Mass for STD-sized Deviations}
\label{sec:analysis}
Our bounds do not sharpen the best known results if $\epsilon$ is large. However, for small $\epsilon$ our bounds become competitive as the number
of samples increase; let us now  compare our bounds (\ref{ud_linear}) and (\ref{ld_linear}) against the existing bounds for this case. 
Here, we focus on missing mass problem(s). We select \cite{Berend_Kontorovich_Bound} for our comparisons since those are the state-of-the-art.
We drop $\epsilon$ in the subscript of $\gamma_{\epsilon}$ in the analysis below.
Despite the fact that the exponent in our bounds is almost linear in $\epsilon$, one can consider the function $\phi$ in
(\ref{phi_func_l}) and imagine rewriting (\ref{ld_linear})
using a functional form that goes like $\exp(- c(.)' \ n {\epsilon}^2)$ instead. Then, the expression for $c'$ would become
\begin{align}
c'(\gamma, \epsilon) = \frac{\color{black}3 \color{black}\phi(\gamma, \epsilon)}{\color{black}8 \color{black}\epsilon} = \frac{3 (\gamma-1)^2}{8 {\gamma}^2 \epsilon \log(\gamma/\epsilon)}.
\end{align}
\noindent Now, remember that we were particularly interested in the case of STD-sized deviations. Since $c'$ is decreasing in $\epsilon$, 
for any $0< \frac{1}{n} < \epsilon < \frac{1}{\sqrt n} < \frac{1}{e}$  we have 
\begin{align}
c'(\gamma, n) \coloneqq
\frac{ 3 \sqrt n (\gamma-1)^2}{8 {\gamma}^2 \log(\sqrt n \gamma)} \leq
c'(\gamma, \epsilon) \leq
\frac{ 3 n (\gamma-1)^2}{8 {\gamma}^2 \log(n \gamma)}.
\end{align}
Optimizing for $\gamma \in D_{\gamma}$ gives 
\begin{align}
\inf_{\gamma} c'(\gamma, \epsilon) = \sup_{\gamma} c'(\gamma, n) =  \frac{ 3 \sqrt n (\gamma_n-1)}{4 {\gamma_n}^2} \color{black}\coloneqq \color{black}c'(n), \label{def_cn}
\end{align}
where $\gamma_n=-2W_{-1} \big(-\frac{1}{2 \sqrt{n e}} \big)$. Therefore, for STD-sized deviations, our 
bound in (\ref{ld_linear}) will resemble $e^{-c'(n) \cdot n \epsilon^2}$. 
We improve their constant for lower deviation which is $\approx 1.92$ as soon as $\color{black}n=1910$\color{black}. 
If we repeat the same procedure for (\ref{ud_linear}), 
it turns out that we also improve their constant for
upper deviation which is $1.0$ as soon as \color{black}$n=427$\color{black}.\\

\noindent Finally, if we plug in the definitions we can see that the following holds for the compensation gap 
\begin{align}
|g(\epsilon)| \leq \sqrt e \cdot \exp \Big( W_{-1}(\frac{-\epsilon}{2 \sqrt e}) \Big), \label{gap_eps}
\end{align}
where we have dropped the subscript of $g$.
It is easy to confirm that the gap is negligible in magnitude for small $\epsilon$ compared to large values of $\epsilon$ in the case of
(\ref{ud_linear}) and (\ref{ld_linear}). This observation supports the
fact that we obtain stronger bounds for small deviations. 

\color{black}

\section{Future Work}
\label{sec:discussion_futurework}
Note that using the notation in (\ref{mcdiarmid_C}), we have $\mathbb{V}[Z] \leq \frac{1}{2} \sum_{i \in \mathcal{S}} c_i^2$ [\cite{efron_stein}] which turns into equality
for sums of independent variables. We would like to obtain using this observation, in the cases where $f$ is any sum over its arguments, for any $\epsilon>0$ the following bounds
\begin{align}
\mathbb{P}(Z-\mathbb{E}[Z] > \epsilon) \leq \exp \Big(-C_3 \cdot \frac{{\epsilon}^2}{V[Z]} \Big) , \nonumber \\
\mathbb{P}(Z-\mathbb{E}[Z] < -\epsilon) \leq \exp \Big(-C_4 \cdot \frac{{\epsilon}^2}{V[Z]} \Big), \label{mcdiarmid_V}
\end{align}
where $V$ is a data-dependent variance-like quantity and $C_3$ and $C_4$ are constants. This can be thought of as a modification of McDiarmid's inequality (appendix \ref{sec:mcdiarmid})
which would then enable us improve our constansts and consequently further sharpen our bounds.

\noindent As future work, we would also like to apply our bounds to \cite{Generalization_Unseen05} so as to
analyze classification error on samples that have not been observed before (i.e. in the training set). 

\section*{Acknowledgement}
The author would like to thank Peter Grünwald who brought the challenge in the missing mass problem(s) to the author's attention, 
shared the initial sketch on how to approach the
problem and provided comments that helped improve an early version of the draft.

\section*{Appendix}

\appendix

\section{Proof of Inequality (\ref{eq:seq_decomp})}
\label{app:prove_decomp}
Assume without loss of generality that
$\cI_b$ has only one element corresponding to $Y_1$ and $\cJ_1=\{1,2\}$ and $k_1=1$ i.e. $w_1$ is split into two parts. Observe that 
deviation probability of $Y$ can be thought of as the total probability mass corresponding to independent Bernoulli variables
$Y_1,...,Y_N$ whose weighted sum is bounded below by some tail size $t$, namely
\begin{align}
\mathbb{P}(Y \geq t) &= \sum_{Y_1,...,Y_N; \ Y \geq t} P(Y_1,...,Y_N) \\
&= \sum_{Y_1,...,Y_N; \ Y''\geq t} \ R(Y_1) \cdot \prod_{i=2}^N R(Y_i) 
+ \sum_{Y_1,...,Y_N; \ Y''< t; \ Y \geq t} R(Y_1) \cdot \prod_{i=2}^N R(Y_i) \\
&= \sum_{Y_1,...,Y_N; \ Y''\geq t} \ R(Y_1) \cdot \prod_{i=2}^N R(Y_i) 
+ \sum_{Y_1,...,Y_N; \ Y''< t; \ Y' \geq t,Y_1 = 1} R(Y_1) \cdot \prod_{i=2}^N R(Y_i) \label{aa}\\
&= \sum_{Y_2,...,Y_N; \ Y''\geq t} \prod_{i=2}^N R(Y_i) 
+ \sum_{Y_1,...,Y_N; \ Y''< t; \ Y' \geq t} q_1 \cdot \prod_{i=2}^N R(Y_i), \label{last_term1}
\end{align}
where $Y''=\sum_{i \geq 2} w_iY_i$ and we have introduced
$R(Y_i)=q_i=q(w_i)$ if $Y_i=1$ and
$R(Y_i)=1-q_i$ if $Y_i=0$.  Here (\ref{aa})
follows because (a) if the condition $Y'' \leq t$ holds  for some $(Y_2,
\ldots, Y_N) = (y_2, \ldots, y_N)$ then clearly it still 
holds for $Y^N = (Y_1, y_2, \ldots, y_N)$ both for $Y_1 =
1$ and for $Y_1 = 0$, and (b) all $Y_1, \ldots, Y_N$ over which the
second sum is taken must clearly have $Y_1 = 1$ (otherwise the
condition $Y'' < t; Y'\geq t$ cannot hold).\\
\noindent Similarly, we can express the upper deviation probability of $Y'$ as
\footnotesize
\begin{align}
\mathbb{P}(Y' \geq t)
&= 
\sum_{Y_1,...,Y_N; \ Y''\geq t} \ R(Y_1) \cdot \prod_{i=2}^N R(Y_i) +
\sum_{Y_{11},Y_{12}, Y_2,...,Y_N; \ Y''< t; \ Y' \geq t}
\Big( R(Y_{11}) \cdot R(Y_{12}) \Big) \prod_{i=2}^N R(Y_i) \\
&= 
\sum_{Y_2,...,Y_N; \ Y''\geq t} \ \prod_{i=2}^N R(Y_i) +
\sum_{Y_{11},Y_{12}, Y_2,...,Y_N; \ Y''< t; \ Y' \geq t}
\Big( R(Y_{11}) \cdot R(Y_{12}) \Big) \prod_{i=2}^N R(Y_i) \\
&\geq  
\sum_{Y_2,...,Y_N; \ Y''\geq t} \ \prod_{i=2}^N R(Y_i) +
\sum_{Y_{11},Y_{12}, Y_2,...,Y_N; \ Y''< t; \ Y' \geq t, Y_{11} = 1, Y_{12} = 1}
\Big( R(Y_{11}) \cdot R(Y_{12}) \Big) \prod_{i=2}^N R(Y_i) \\ &= 
\sum_{Y_2,...,Y_N; \ Y''\geq t} \ \prod_{i=2}^N R(Y_i) +
\sum_{Y_{11},Y_{12}, Y_2,...,Y_N; \ Y''< t; \ Y' \geq t, Y_{11} = 1, Y_{12} = 1}
(q_{11} \cdot q_{12}) \prod_{i=2}^N R(Y_i), \label{last_term2}
\end{align}
\normalsize
where $R(Y_{ij})=q_{ij}=q^{\circ}(w_{ij})$ if $Y_{ij}=1$ and
$R(Y_{ij})=1-q_{ij}=1-q(w_{ij})$ otherwise.  Therefore, combining (\ref{last_term1}) and (\ref{last_term2}) we
have that
\begin{align}
  \mathbb{P}(Y' \geq t) - \mathbb{P}(Y \geq t) \geq
  \sum_{Y_1,...,Y_N; \ Y''< t; \ Y' \geq t} (q_{11} \cdot
  q_{12} - q_1) \prod_{i=2}^N R(Y_i). \label{eq:diff_prob}
\end{align}
In order to establish (\ref{eq:seq_decomp}), we require 
the expression for the difference between deviation probabilities in (\ref{eq:diff_prob}) to be non-negative for all \color{black}$t>0$ \color{black}which
holds when $q_1 \leq q_{11} \cdot q_{12}$ i.e. under condition $(b)$. For the missing mass
, condition $(b)$ holds. Suppose, without loss of generality, that $w_i$ is split into two terms; namely, we have $w_i=w_{ij} + w_{ij'}$. Then, one can verify the condition as follows
\begin{align}
 q(w_i) = (1- w_i)^n \leq (1-w_{ij})^n \cdot (1-w_{ij'})^n \nonumber \\
= {\Big( 1-\underbrace{(w_{ij}+w_{ij'})}_{w_i} + \underbrace{w_{ij} \cdot w_{ij'}}_{\geq 0} \Big)}^n.
\end{align}
The proof follows by induction. Finally, choosing tail size $t=\epsilon+\mathbb{E}Y$ implies the result. 

\section{Generalization to Unit Interval}
\label{generalization_to_bounded}
The following lemma shows that any result for mixture of independent
Bernoulli variables extends to mixture of independent
variables with smaller or equal means defined on the unit interval.
\begin{lemma} 
\upshape Let $\mathbb{S}$ be some countable set and consider independent random variables $\{Z_i\}_{i\in \mathbb{S}}$ that belong to $[0,1]$ with probability one and
independent Bernoulli random variables $\{{Z'}_i\}_{i\in \mathbb{S}}$ 
such that $\mathbb{E}[Z_i] \leq \mathbb{E}[{Z'}_i]$ almost surely for all $i\in \mathbb{S}$.
In addition, let $0 \leq \{w_i\}_{i \in \mathbb{S}} \leq 1$ be their associated weights. Then, the mixture random variable
$Z=\sum_{i\in \mathbb{S}} w_i Z_i$ is more concentrated than the mixture random variable $Z'=\sum_{i\in \mathbb{S}} w_i Z_i'$
for any such $Z$ and $Z'$.
\end{lemma}
\begin{proof}
 For any fixed $t>0$ and for $\lambda>0$, applying Chernoff's method to any non-negative random variable $Z$ we obtain
\begin{align}
 \mathbb{P}(Z\geq t) \leq \frac{\mathbb{E}[e^{\lambda Z}]}{e^{\lambda t}}. \label{Chernoff}
\end{align}
Now, observe that for any convex real-valued function $f$ with $D_f=[0,1]$, we have that $f(z) \leq (1-z)f(0)+zf(1)$ for
any $z \in D_f$. Therefore, with $f$ chosen to be the exponential function, for all $i \in \cI$ we will have the following
\begin{align}
 \mathbb{E}[e^{\lambda w_i Z_i}] \leq \mathbb{E}[(1-w_i Z_i)+w_i e^{\lambda w_i} Z_i] \leq \\
 (1-\mathbb{E}[Z'_i])+e^{\lambda w_i}\mathbb{E}[Z_i'] +
 \underbrace{(1-e^{\lambda})}_{<0} \underbrace{(1-w_i)}_{\geq 0} \underbrace{\mathbb{E}[Z_i']}_{>0} \leq \\
 (1-\mathbb{E}[Z'_i])+e^{\lambda w_i}\mathbb{E}[Z_i'] = \mathbb{E}[e^{\lambda w_i Z_i'}].
 \label{weighted_bernoulli}
\end{align}
Note that we can apply Chernoff to $Z'$ as well. In order to complete the proof, it is sufficient to establish
that the RHS of (\ref{Chernoff}) is smaller for $Z$ compared to $Z'$. This follows since we have
\begin{align}
 \mathbb{E}[e^{\lambda Z}] = \mathbb{E}[\prod_{i\in \cI} e^{\lambda w_i Z_i}] 
 = \prod_{i\in \cI} \mathbb{E}[e^{\lambda w_i Z_i}] \label{indep_assump} \\
 \leq \prod_{i\in \cI} \mathbb{E}[e^{\lambda w_i Z_i'}] =
 \mathbb{E}[\prod_{i\in \cI} e^{\lambda w_i Z_i'}] = \mathbb{E}[e^{\lambda Z'}]. \label{convexity}
\end{align}
Here, (\ref{indep_assump}) follows because of 
independence whereas the inequality in (\ref{convexity}) is due to convexity just as concluded in (\ref{weighted_bernoulli}). Finally,
the last step holds again since the variables are independent.
\end{proof}

\begin{theorem} \ {\bf [Bernstein]} \label{thm:bernstein} \upshape
Let $Z_1, ..., Z_N$ be independent zero-mean random variables such that $|Z_i| \leq \alpha$ almost surely for 
all $i$. Then, using Bernstein's inequality (\cite{Bernstein_ineq}) one obtains for all $\epsilon>0$:
\begin{equation}\label{eq:bernsteina}
\mathbb{P} ( \sum_{i=1}^{N} Z_i > \epsilon )
 \leq \exp \Big(-\frac{{\epsilon}^2}{2(V + \frac{1}{3} \alpha \epsilon)}\Big),
\end{equation}
where $V=\sum_{i=1}^{N}\mathbb{E}[{Z_i}^2]$. 
\end{theorem}
\noindent Now if we consider the sample average 
$\bar{Z} = n^{-1}\sum_{i=1}^n Z_i$, and let $\bar{\sigma}^2$ be the
average sample variance of the $Z_i$, i.e. $\bar{\sigma}^2 := n^{-1} \sum_{i=1}^n \text{\sc var\;}[{Z_i}]=n^{-1} \sum_{i=1}^n E[{Z_i}^2]$. 
Using (\ref{eq:bernsteina})
with $n \cdot \epsilon$ in the role of
$\epsilon$, we get 
\begin{equation}\label{eq:bernsteinb}
\mathbb{P} ( \bar{Z} > \epsilon )
 \leq \exp \Big(-\frac{n {\epsilon}^2}{2(\bar{\sigma}^2 + \frac{1}{3} \alpha \epsilon)}\Big).
\end{equation}
If $Z_i$s are, moreover, not just independent but also
identically distributed, then $\bar{\sigma}^2$ is equal to $\sigma^2$ i.e.
the variance of $Z$. The latter presentation makes explicit: (1) the exponential decay with $n$; (2) the fact
that for $\bar{\sigma}^2 \leq \epsilon$ we get a tail probability with
exponent of order $n \epsilon$ rather than $n \epsilon^2$ \cite{Lugosi_Concentration,boucheron2013concentration} 
which yields stronger bounds for small $\epsilon$.

\section{McDiarmid's Inequality}
\label{sec:mcdiarmid}
\begin{theorem} \ {\bf [McDiarmid]} \label{thm:mcdiarmid} \upshape
Let $X_1, ..., X_m$ be independent random variables belonging to some set $\mathcal{X}$ and let $f: \ {\mathcal{X}}^m \rightarrow \mathbb{R}$
be a measurable function of these variables. Introduce independent shadow variables $X_1', ..., X_m'$ as well as the notations
$Z=f(X_1,...,X_{i-1},X_i,X_{i+1},...,X_m)$ and $Z_i'=f(X_1,...,X_{i-1},X_i',X_{i+1},...,X_m)$. 
Suppose that for all $i \in \mathcal{S}$ (with $|\mathcal{S}|=m$) and for all realizations $x_1,...,x_m,x_i' \in \mathcal{X}$, $f$ satisfies 
\begin{align}
\lvert z - z_i'\rvert = \lvert f(x_1,...,x_{i-1},x_i,x_{i+1},...,x_m)-f(x_1,...,x_{i-1},x_i',x_{i+1},...,x_m)\rvert \leq c_i.
\end{align}
Setting $C = \sum_{i \in \mathcal{S}} c_i^2$, for any $\epsilon>0$ one obtains  [\cite{mcdiarmid89}]
\begin{align}
&\mathbb{P}(Z-\mathbb{E}[Z] > \epsilon) \leq \exp \Big(-\frac{2 {\epsilon}^2}{C} \Big),& \nonumber \\
&\mathbb{P}(Z-\mathbb{E}[Z] < -\epsilon) \leq \exp \Big(-\frac{2 {\epsilon}^2}{C} \Big).& \label{mcdiarmid_C}
\end{align}
\end{theorem}

\bibliography{missing_mass}
\bibliographystyle{plainnat}
\end{document}